\documentclass{article}

\usepackage{xcolor}
\usepackage[table]{xcolor}
\usepackage{microtype}
\usepackage{graphicx}
\usepackage{subcaption}
\usepackage{booktabs} % for professional tables
\usepackage{comment} % for comments environment
\usepackage{float}
\usepackage{hyperref}
\usepackage{verbatim}
\usepackage{multirow}
\usepackage{tcolorbox}
\usepackage{multicol}
\usepackage[preprint]{icml2026}

\usepackage{amsmath}
\usepackage{amssymb}
\usepackage{mathtools}
\usepackage{amsthm}
\usepackage{amsmath, amsthm, amssymb, amsfonts}

\usepackage[capitalize,noabbrev]{cleveref}

\theoremstyle{plain}
\newtheorem{theorem}{Theorem}[section]

\theoremstyle{definition}
\newtheorem{definition}[theorem]{Definition}
\newtheorem{assumption}[theorem]{Assumption}
\theoremstyle{remark}

\newcommand{\yang}[1]{\textcolor{red}{\small{\bf [Yang: #1]}}} 
 
\usepackage[textsize=tiny]{todonotes}
\usepackage{comment}
\icmltitlerunning{Submission and Formatting Instructions for ICML 2026}

\begin{document}

\twocolumn[
  \icmltitle{MolEvolve: LLM-Guided Evolutionary Search for Interpretable Molecular Optimization}



  \icmlsetsymbol{equal}{*}

  \begin{icmlauthorlist}
    \icmlauthor{Xiangsen Chen}{poly}
    \icmlauthor{Ruilong Wu}{ust}
    \icmlauthor{Yanyan Lan}{air}
    \icmlauthor{Ting Ma}{hit}
    \icmlauthor{Yang Liu}{poly}
    %\icmlauthor{Firstname6 Lastname6}{sch,yyy,comp}
    %\icmlauthor{Firstname7 Lastname7}{comp}
    %\icmlauthor{Firstname8 Lastname8}{yyy,comp}
  \end{icmlauthorlist}

  \icmlaffiliation{ust}{Hong Kong University of Science and Technology(Guangzhou)}
  \icmlaffiliation{poly}{Hong Kong Polytechnic University}
  \icmlaffiliation{air}{Tsinghua University}
  \icmlaffiliation{hit}{Harbin Institute of Technology (Shenzhen)}

  \icmlcorrespondingauthor{}{yang-veronica.liu@polyu.edu.hk}

  \icmlkeywords{Machine Learning, ICML}

  \vskip 0.3in
]



\printAffiliationsAndNotice{}  % no special notice (required even if empty)

\begin{abstract}
Despite deep learning’s success in chemistry, its impact is hindered by a lack of interpretability and an inability to resolve 'activity cliffs', where minor structural nuances trigger drastic property shifts. Current representation learning, bound by the similarity principle, often fails to capture these structural-activity discontinuities. To address this, we introduce MolEvolve, an evolutionary framework that reformulates molecular discovery as an autonomous, look-ahead planning problem.
Unlike traditional methods that depend on human-engineered features or rigid prior knowledge, MolEvolve leverages a Large Language Model (LLM) to actively explore and evolve a library of executable chemical symbolic operations. By utilizing the LLM to cold start and an Monte Carlo Tree Search (MCTS) engine for test-time planning with external tools (e.g. RDKit), the system self-discovers optimal trajectories autonomously. This process evolves transparent reasoning chains that translate complex structural transformations into actionable, human-readable chemical insights. Experimental results demonstrate that MolEvolve’s autonomous search not only evolves transparent, human-readable chemical insights, but also outperforms baselines in both property prediction and molecule optimization tasks.
%\yang{Despite deep learning’s success in chemistry, its impact is hindered by a lack of interpretability and an inability to resolve 'activity cliffs', where minor structural nuances trigger drastic property shifts. Current representation learning, bound by the similarity principle, often fails to capture these structural-activity discontinuities. To address this, we introduce MolEvolve, an evolutionary framework that reformulates molecular discovery as an autonomous, look-ahead planning problem.
%Unlike traditional methods that depend on human-engineered features or rigid prior knowledge, MolEvolve leverages a Large Language Model (LLM) to actively explore and evolve a library of executable chemical symbolic operations. By utilizing the LLM to cold start and an Monte Carlo Tree Search (MCTS) engine for test-time planning with external tools (e.g. RDKit), the system self-discovers optimal search trajectories without human intervention \xiangsen{Replaced "without human intervention" to avoid potential challenges from reviewers, since we use prompts for the cold start.}. This process evolves transparent reasoning chains that translate complex structural transformations into actionable, human-readable chemical insights. Experimental results demonstrate that MolEvolve’s autonomous search not only evolves transparent, human-readable chemical insights, but also outperforms baselines in both property prediction and molecule optimization tasks.}
\end{abstract}

\section{Introduction}
\label{sec:intro}
%\textcolor{blue}{[lack a challenge image]}
Deep learning has become a paradigm shift in molecule-related tasks from manual experimentation to data-centric discovery. In this domain, deep learning models, especially Graph Neural Networks (GNNs), have demonstrated remarkable performance by modeling the geometric topology of molecules \cite{stark20223d,li2023interpretable}. 
Techniques such as pre-training on molecular data (e.g., atoms and bonds) have further solidified GNNs as the standard for diverse property prediction tasks \cite{stuyver2022quantum,cremer2023equivariant}. More recently, Large Language Models (LLMs) show great promise in scientific tasks. Domain-specific LLMs \cite{taylor2022galactica,liu2024moleculargpt} and Retrieval-Augmented Generation (RAG) methods \cite{xian2025molrag,zhang2025automated} leverage parametric domain knowledge and vast corpora of external literature to capture knowledge and semantic understanding that deep learning models may miss.

\begin{figure*}[t]
    \centering
    \includegraphics[width=\textwidth]{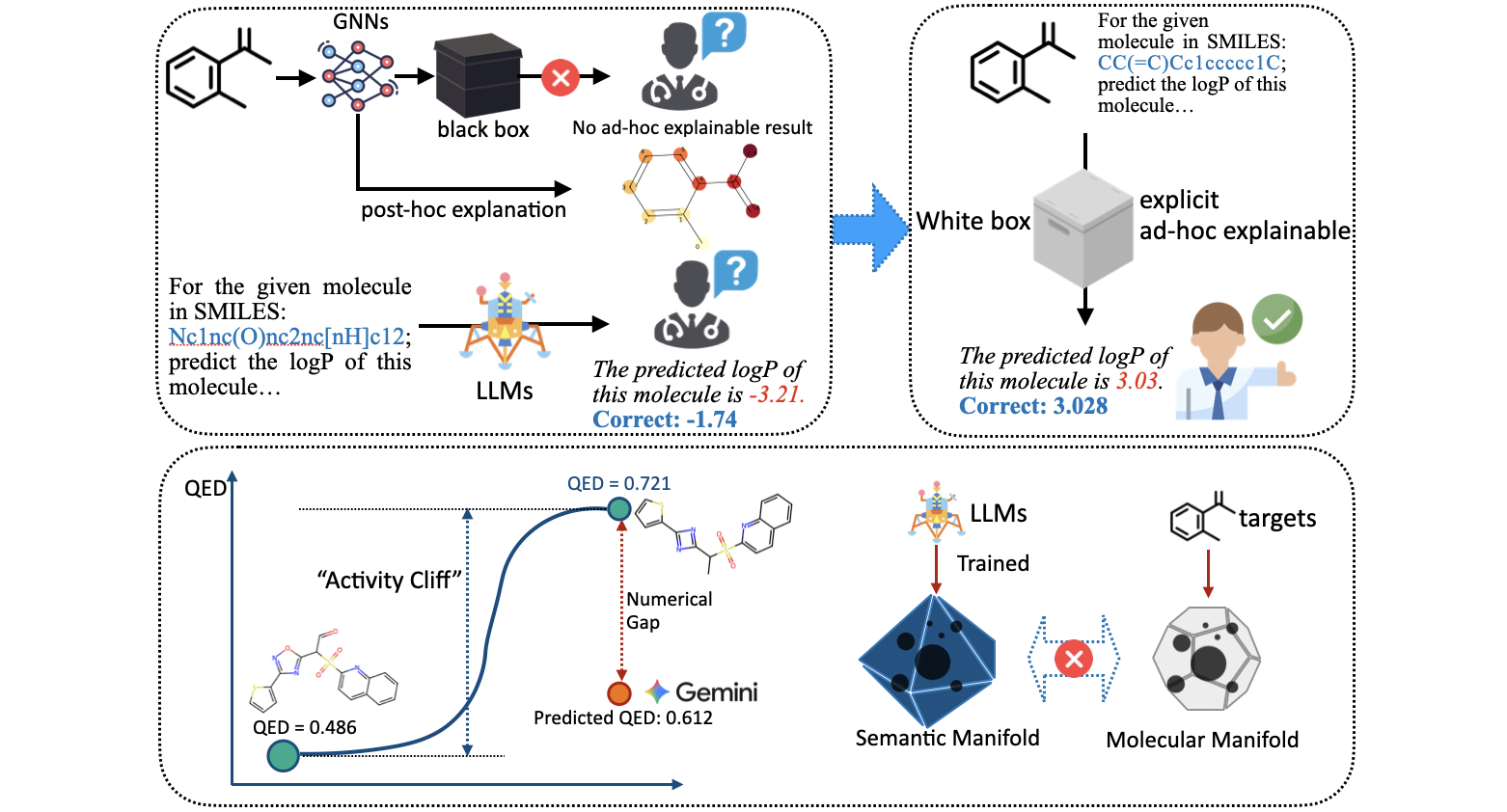}
    \caption{(Up) Existing GNNs act as black boxes with post-hoc interpretations, while LLMs suffer from numerical hallucinations on semantic manifolds. (Down) The representation-precision paradox: minor structural distances lead to "activity cliffs" in molecular manifolds, which LLMs fail to capture precisely due to semantic smoothing.}
    \label{fig:challenge}
\end{figure*} 

In spite of their capabilities, these existing methods face significant limitations. Deep learning models like GNNs operate as black boxes despite their powerful capabilities in molecular data mining \cite{li2023interpretable, lavecchia2025explainable, stegmann2023trustworthy}. They are often unable to provide accurate feature attributions to explain which specific substructures contribute to a given molecular property. This reduces the scientific interpretability of the models. In addition, GNN-based explanations (e.g., attention weights) serve as post-hoc approximations rather than intrinsic logic \cite{stegmann2023trustworthy, he2024explaining}. On the other hand, LLMs, even when fine-tuned on domain-specific knowledge, are prone to hallucination \cite{suzgun2025language}. LLMs are prone to generate plausible but factually incorrect scientific knowledge or properties \cite{truhn2023large}, and they often struggle to predict accurate quantitative results \cite{wang2025evaluating}. %Meanwhile, LLM-centric methods often rely on LLMs' generated output, failing to verify whether the generated rationale mathematically aligns with the target property. 

%We lack a method that can leverage data-driven precision while expressing predictive logic through explicit, executable chemical rules. 
Previous approaches often treat predictive precision and interpretability as separate objectives \cite{oviedo2022interpretable}. Consequently, a critical gap remains: scientists need to know not just the final results, but the explicit causal rules that enable a molecule to possess a certain property \cite{mengaldo2024explain}. 
We argue that the difficulty in achieving such interpretable precision arises from a fundamental representation-precision paradox.
As demonstrated in Appendix~\ref{appendix:theoretical_analysis}
%A fundamental missing piece is a mechanism to actively search for and verify explicit chemical rules, which transforms feature engineering from a static selection process into a dynamic, goal-oriented reasoning task. However, achieving such interpretable precision faces representation-precision paradox. 
and Figure~\ref{fig:challenge}, molecular properties often exhibit high Lipschitz constants, known as "activity cliffs," where minor structural nuances trigger drastic property shifts.
This challenge arises from the severe topological mismatch between molecular manifolds and continuous semantic manifolds. Recent work suggests that LLM internal representations are best characterized as manifolds where proximity reflects semantic similarity, rather than the precise Euclidean distances required for physical properties \cite{modell2025origins}. LLMs attempt to map discrete chemical tokens onto continuous physical manifolds, which causes them to struggle with precise quantitative tasks \cite{wang2025evaluating, truhn2023large}. 
While GNNs can fit these manifolds, their reasoning remains an opaque black box \cite{stegmann2023trustworthy}. 
To resolve this paradox, we argue that one must transcend continuous embeddings and operate within a space of executable symbolic logic. Unlike smooth vectors, discrete symbolic rules can naturally model non-linear discontinuities (e.g., logic branches).
A fundamental missing piece is a mechanism to actively search for and verify explicit chemical rules, thereby transforming chemical feature engineering into a dynamic, goal-directed reasoning task. 

%We argue that the difficulty in achieving such interpretable precision arises from a fundamental representation-precision paradox.
%A fundamental missing piece is a mechanism to actively search for and verify explicit chemical rules, which transforms feature engineering from a static selection process into a dynamic, goal-oriented reasoning task. However, achieving such interpretable precision faces representation-precision paradox. 
%This challenge arises from the severe topological mismatch between discrete molecular tokens and continuous physical manifolds. Recent work suggests that LLM internal representations are best characterized as manifolds where proximity reflects semantic similarity, rather than the precise Euclidean distances required for physical properties \cite{modell2025origins}.
%As shown in Figure~\ref{fig:challenge}, molecular properties often exhibit high Lipschitz constants, known as "activity cliffs," where minor structural nuances trigger drastic property shifts.
%While GNNs can fit these manifolds, their reasoning remains an opaque black box \cite{stegmann2023trustworthy}. LLMs attempt to map discrete chemical tokens onto continuous physical manifolds, which causes them to struggle with precise quantitative tasks \cite{wang2025evaluating, truhn2023large}. 

To address the challenge, we propose the MolEvolve framework.
% an LLM-augmented Monte Carlo Tree Search (MCTS) framework. 
Unlike simulations of fine-tuned LLMs or opaque GNN approximations, MolEvolve reformulates molecular discovery as explicit look-ahead planning over a space of executable chemical symbolic operations. Instead of outputting a final answer directly, the system actively constructs an explicit path.
%Our framework reformulates generating molecular property information as a search problem over the space of executable chemical rules.
First, the framework initiates with an LLM-based cold start. It functions as a computational chemist to diagnose the molecule and propose initial heuristic rules (e.g., logP for solubility calculation). It distills symbolic priors to seed the search.
Then, we apply external chemical tools to generate corresponding rules. These rules serve as the foundation for our initialized MCTS search tree. 
The MCTS tree navigates the vast space of executable chemical rules guided by an LLM that proposes candidate operations for diverse downstream tasks (e.g. \textit{"adding a hydroxyl group"} for optimization or \textit{"calculating the logP"}) with RDKit \cite{bento2020rdkit} as external tool. 
Crucially, we enforce a closed-loop verification mechanism: every proposed operation is rigorously evaluated by external tools during the search. This allows the model to discover the consequences of a rule and prune invalid or low-fidelity pathways before commitment. By grounding semantic reasoning in verifiable execution, our framework evolves transparent, high-precision chemical insights.
%Crucially, we employ a closed-loop mechanism where the system actively identifies features and tasks the LLM to iteratively evolve the features to capture missing structural insights. 
%And finally the framework constructs an explicit Monte Carlo tree, which represents the final selected features. 
%Therefore, our framework combines the reasoning power of LLMs with rigorous verification to significantly enhance scientific interpretability. 
We conduct quantitative experiments on different molecule-related tasks: property prediction and optimization. The results demonstrate that MolEvolve outperforms both GNNs and LLM-based methods in property prediction; We also show evidence that MolEvolve outperforms domain-specific models and vanilla strong LLMs. 
%\yang{not sure what is "vanilla strong LLMs with a small LLM as the base model", change it to just "vanilla strong LLMs"}\xiangsen{Fixed.}

Our main contributions are summarized as follows: 
\begin{itemize} 
\item We identify the challenge in achieving interpretable precision of molecular property, and propose MolEvolve to address it via test-time symbolic planning, bridging the gap between semantic intuition and physical rigor.
\item We propose a closed-loop mechanism where an LLM acts as a molecule operator to iteratively evolve and refine chemical descriptors based on feedback. 
\item Extensive experiments on benchmarks demonstrate that our method outperforms specialized GNNs and general-purpose LLMs, while providing human-readable chemical insights. 
\end{itemize}

\section{Related Works}
\paragraph{GNN for Molecules.}
% Graph Neural Networks (GNNs) evolves from local message-passing to sophisticated foundation models in molecular representation learning \cite{gilmer2020message}. 
% Graph Neural Networks (GNNs) evolves rapidly in molecular representation learning \cite{gilmer2020message}. 
%have established themselves as the de facto standard for molecular representation learning, evolving from local message-passing to sophisticated foundation models \cite{gilmer2020message}. 
GNNs have been developed to aggregate atomic features and bond interactions \cite{gilmer2020message,xu2018gin, yang2021graphformers,wang2022molecular}. 
Furthermore, pre-training strategies are capable of both learning representations from millions of molecular graphs \cite{zhou2023uni, hu2019strategies, rong2020self} and increasing performance in data-scarce scenarios \cite{gao2024ssgnn}.
%Furthermore, pre-training on millions of molecular graphs have enabled GNNs to learn representations that transfer to various downstream tasks \cite{zhou2023uni, hu2019strategies, rong2020self}. 
%In addition, self-supervised frameworks have shown that multi-level pre-training can increase performance in data-scarce scenarios \cite{gao2024ssgnn}.
% Recent methods employ graph structure learning and scaling behaviors to boost the performance across benchmarks like MoleculeNet \cite{wu2018moleculenet, van2022MoleculeACE}. 
%To improve data efficiency, self-supervised frameworks for specific chemical domains, such as polymers \cite{gao2024ssgnn}, have shown that multi-level pre-training (node, edge, and graph) can increase performance in data-scarce scenarios. 
Despite their success in benchmarks \cite{wu2018moleculenet, van2022MoleculeACE}, the inherent opacity of GNNs remains a significant problem in explainability. 
% Despite these successes, the inherent opacity of GNNs remains a critical bottleneck. 
% \textcolor{red}{[to\_refine]}
While recent efforts %in Explainable AI (XAI) for science 
attempt to elucidate GNN decisions via subgraph and fragment identification \cite{panapitiya2024fragnet}, attention mechanisms \cite{li2023interpretable}, or attribution generation \cite{chen2025aces,he2024explaining}, these explanations are predominantly post-hoc approximations \cite{ying2019gnnexplainer}.
%They highlight input regions (e.g., atoms or bonds) that influence the prediction but fail to provide ad-hoc, explicit, chemical rule-based explainable rationales (e.g., "this molecule is soluble because of the specific ratio of polar groups"). 
%Unlike these black-box embeddings, our framework seeks to construct explicit, executable chemical rules through a search-based paradigm, ensuring that the decision process is ad-hoc, transparent and chemically grounded by design.
% \textcolor{red}{[to\_add] In addition, the GNN models will over-smooth out the small structural difference during information aggregation. }
In addition, the GNN models will over-smooth out the small structural difference during information aggregation.
They face a critical challenge known as "activity cliffs" \cite{van2022MoleculeACE}. 
%GNNs often fail to capture these discontinuities and exhibit high errors on cliff-forming molecules, since they opt to over-smoothing embeddings in message passing . 
To mitigate this, recent works have focused on refining continuous representations on contrastive learning \cite{zhang2024molemcl, shen2023online}, multi-grained perception \cite{shu2025mtpnet} and weighted graphs \cite{chen2025mapcliff}. 
Despite these advances, these methods predominantly focus on patching the continuous embedding space. However, they remain constrained by the underlying message-passing paradigm, where small structural differences are inevitably over-smoothed out during information aggregation.

\paragraph{LLMs for Molecules.}
The success of LLMs has catalyzed a surge of interest in adapting them for scientific domains. 
In order to make LLMs understand molecules, early works focused on treating molecular strings as language \cite{weininger1988smiles, krenn2022selfies}. 
With this, fine-tuned LLMs are applied \cite{raffel2020t5,touvron2023llama} for molecule generation and property prediction \cite{taylor2022galactica, liu2024moleculargpt, wu2025uni}. 
More recently, the paradigm has shifted towards LLM agents that utilize external tools. These agents, like ChemCrow \cite{bran2023chemcrow} empower LLMs in chemical sciences by multi-agent collaboration or invoking chemistry-related scripts or APIs \cite{wu2025chemagent,ghafarollahi2025sciagents}. 
To mitigate the lack of domain knowledge, RAG-based methods \cite{xian2025molrag, zhang2025automated} integrate external scientific literature into the LLM's context window to enhance reasoning capabilities. 
%\textcolor{blue}{to\_modify] 
However, as highlighted in Section~\ref{sec:intro}, LLMs are prone to molecular numerical hallucinations because they rely on the continuous manifold assumption \cite{wang2025evaluating, truhn2023large}. 
Existing works attempt to mitigate this by evolving from linear SMILES modeling to multi-view multi-modal frameworks \cite{ju2025m2llm}. 
% In addition, employing optimized prompts and augmented generation to constrain LLM hallucinations \cite{reed2025augmented} and utilizing semi-supervised learning with curriculum strategies \cite{ijcaiWu24Semisupervised} to enhance activity cliff estimation.
To mitigate hallucinations and enhance activity cliff estimation, recent works employ strategies from optimized prompting \cite{reed2025augmented} to semi-supervised curriculum learning \cite{ijcaiWu24Semisupervised}.
In addition, frameworks seek to enrich molecular representations by using LLM-generated textual explanations to guide representation learning via contrastive objectives \cite{wang2024efficient} or employing instruction tuning to align LLMs with molecular graph modalities \cite{yu2025collaborative}.
%TextMOL \cite{wang2024efficient} uses LLM-generated textual explanations to guide representation learning via contrastive objectives, while LlaMo \cite{yu2025collaborative} employs instruction tuning to align LLMs with molecular graph modalities.
Despite these methodological advancements, these approaches predominantly focus on enhancing LLMs through adding views, tuning prompts, or augmenting data, without addressing the fundamental topological mismatch of the continuous embedding space. 

\section{Method}
We demonstrate MolEvolve in this chapter. As illustrated in Figure~\ref{fig:framework_arch}, MolEvolve reframes molecular engineering as a directed search problem over an executable symbolic chemical space. 
\begin{figure*}[t]
    \centering
    \includegraphics[width=0.95\textwidth]{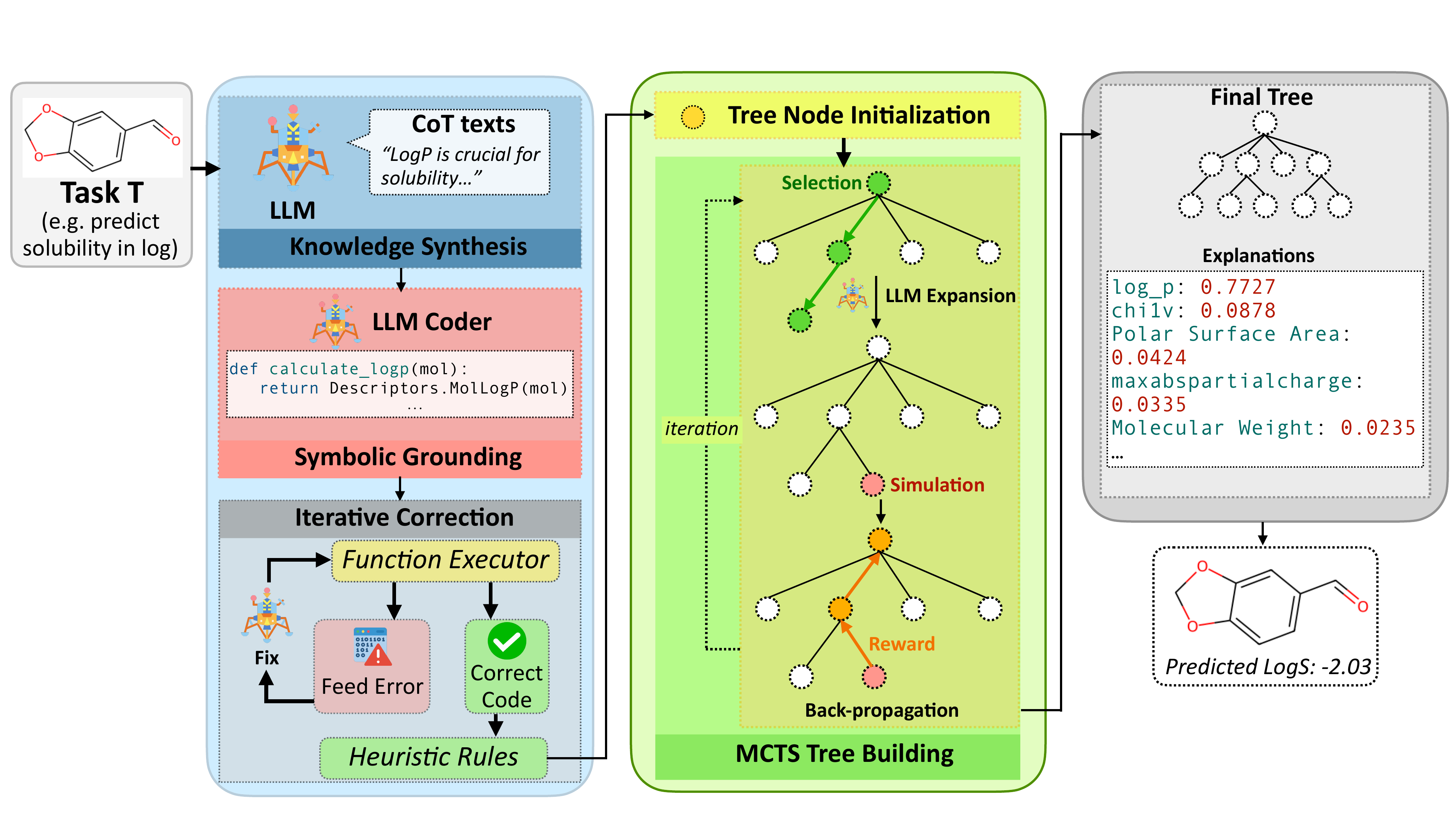}
    \caption{The overall architecture of MolEvolve. Phase 1 (Cold Start) distills domain knowledge into executable heuristic rules via symbolic grounding and self-correction. Phase 2 (LLM-MCTS) utilizes these rules to initialize an evolutionary search tree, where the LLM acts as molecule operator to guide selection and expansion within a rigorous verification loop.}
    \label{fig:framework_arch}
\end{figure*}

\subsection{Cold-start for Symbolic Knowledge Distillation}
\textbf{Overview.} Exploring the discrete and large molecular space presents a fundamental challenge for heuristic-free search algorithms. In particular, standard MCTS suffers from prohibitive sample complexity during its initialization phase due to the absence of effective guidance.
To address this limitation, the cold start mechanism introduces informative priors. Rather than discovering chemical rules from scratch via trial-and-error, our approach explicitly encodes these rules into the search process.
As illustrated in Figure~\ref{fig:framework_arch}, our cold start mechanism operates as a symbolic knowledge distillation process. It consists of three stages: (1) \textbf{Knowledge Synthesis}, where abstract chemical rules are analyzed by LLMs; (2) \textbf{Symbolic Grounding}, where rules are translated into executable code; and (3) \textbf{Iterative Self-Correction}, which ensures the validity of the generated logic through execution feedback. \\
\textbf{Knowledge Synthesis.} Directly asking an LLM to write code often yields incorrect or irrelevant functions. To mitigate this, we first employ CoT \cite{wei2022chain} prompting strategy to synthesize text-based prior knowledge. Given the task $\mathcal{T}$, we use task-specific prompt $\mathcal{P}_0(\mathcal{T})$ to identify critical molecular substructures or physicochemical properties associated with the target label. 
Formally, this step produces a set of textual rules $\mathcal{K}_{text} = \text{LLM}(\mathcal{P}_0(\mathcal{T}))$. This textual representation ensures the subsequent code generation is grounded in chemical logic rather than mere syntax completion. An example can be viewed in Appendix~\ref{appendix:cold_start_example}. 
%\yang{should there be a prompt to transform task $\mathcal{T}$, such as $\mathcal{K}_{text} = \text{LLM}(\mathcal{P}_0(\mathcal{T}))$? In appendix, provide examples for $\mathcal{T}$, $\mathcal{K}_{text}$ and necessary prompt structure, and refer it here so that it is clear what you meant.}\xiangsen{Added prompt denotation.}

\textbf{Symbolic Grounding via Code Generation.} To make $\mathcal{K}_{text}$ 
%\yang{is this the same as $\mathcal{K}_{prior}$ mentioned in discussion?}\xiangsen{$\mathcal{K}_{prior}$ of discussion part has been deleted for better clarification}
usable by MCTS, we must ground natural language into symbolic representations. 
%We employ the LLM as coder to translate $\mathcal{K}_{text}$ into Python functions $\Phi$.
Given a molecule SMILES string $S$, for each rule $\mathcal{K}_{text}^{(i)} \in \mathcal{K}_{text}$, we employ LLM as coder to translate it into Python function $\phi_i \in \Phi$ with standard chemical libraries like RDKit. We denote this as $\Phi = \{ \phi_i = \text{LLM}(\mathcal{K}_{text}^{(i)}) \}_{\mathcal{K}_{text}^{(i)} \in \mathcal{K}_{text}}$.
% Specifically, we instruct the model to implement each rule in $\mathcal{K}_{text}$ using standard chemical libraries like RDKit. 
This grounded representation $\Phi$ enables rigorous mathematical operations in the subsequent search process. %\yang{add equation here.}

\textbf{Execution-based Iterative Self-Correction.} A critical challenge in LLM-based coding is the generation of non-executable buggy code. We propose a closed-loop self-correction mechanism. For each generated function $\phi_i$, we attempt to execute it. If an exception occurs, we capture the error trace and feed it back to the LLM:\begin{equation}\phi_i^{t+1} \leftarrow \text{LLM}(\phi_i^t, e, \mathcal{P}_{fix})\end{equation}where $\mathcal{P}_{fix}$ is a rectification prompt. This loop repeats until $\phi_i$ executes successfully or a retry limit is reached.

\subsection{LLM-guided Evolutionary Search}
To evolutionarily explore the chemical space, we adopt a modified MCTS framework that integrates symbolic knowledge distilled from the cold start phase. Unlike conventional MCTS approaches that rely on random rollouts or purely data-driven value networks, our framework leverages LLMs to guide the search process for ensuring both chemical validity and boosting search space. %We define a search state $s$ as a tuple $(G, \mathcal{F})$, where $G$ denotes the molecular graph and $\mathcal{F}$ represents the accumulated set of features or constraints. 
We define a search state $s$ as a tuple of features $\mathcal{F}$ (in molecule prediction) or a molecular SMILES string $S$ (in molecular optimization). Importantly, the initial state $s_0$ is constructed with the symbolic functions $\Phi$ generated during the cold start phase, which provides an initialization with domain-specific inductive knowledge. 
%\yang{what is the exactly meaning for "augmented" here?  }
Our MCTS follows the standard four-stage procedure: \textit{Selection}, \textit{Expansion}, \textit{Simulation}, and \textit{Backpropagation}. We introduce key modifications to the \textit{Expansion} and \textit{Simulation} stages to explicitly incorporate symbolic knowledge.
A detailed formal comparison of detailed formulations of MCTS stage for tasks selected in our paper is provided in Appendix \ref{appendix:mcts_definitions} (see Table \ref{tab:mcts_definition_appendix}).

\textbf{Selection.}
We adopt the Upper Confidence Bound for Trees (UCT) strategy to traverse the tree from the root to a leaf node, balancing exploration and exploitation:
\begin{equation}
\label{eq:uct}
a_t = \arg\max_{a \in \mathcal{A}(s)} 
\left(
Q(s,a) + c \cdot \sqrt{\frac{\ln N(s)}{N(s,a)}}
\right),
\end{equation}
where $Q(s,a)$ is the estimated value of action $a$ in state $s$, $N(s)$ is the visit count of state $s$, $N(s,a)$ is the visit count of action $a$, and $c$ is an exploration constant.

\textbf{LLM-Guided Node Expansion.}
Instead of exhaustively enumerating all possible chemical modifications, we employ the LLM as a generative policy network conditioned on the symbolic priors. 
% Instead of exhaustively enumerating all possible chemical modifications or feature combinations, which is computationally infeasible, we employ the LLM as a generative policy network during the expansion phase.
% Given a leaf state $s_L$, the LLM is prompted with the textual diagnosis $D$ \yang{confusing.. no $D$ is found in cold-start phase. notation needs to be consistent.} derived in the cold start phase (e.g., specific structural liabilities or optimization directions) to propose a set of high-potential actions. 
Given a leaf state $s_L$, the LLM is prompted with the textual rules $\mathcal{K}_{text}$ derived in the cold start phase (e.g., specific structural liabilities or optimization directions) to propose a set of high-potential actions. 
This design ensures that the expansion is not a random exploration but a focused execution of the diagnostic logic established in cold start phase, thereby integrating the cold-start knowledge directly into the evolutionary search loop.
% Given a leaf state $s_L$, the LLM proposes a limited set of high-potential actions $\mathcal{A}_{\mathrm{LLM}}(s_L)$.\yang{is this part similar or same to 4.1? It is better than 4.1 is still part of the 4.2 to provide a holistic solution.}

\textbf{Simulation.}
Instead of performing deep Monte Carlo rollouts which are computationally expensive, we immediately evaluate the expanded node using the composite reward function. The reward function integrates task-specific objectives with symbolic heuristic guidance from $\Phi$. It consists of three distinct objectives: the primary task performance ($R_{\text{task}}$), the alignment with symbolic priors ($R_{\text{heuristic}}$), and structural constraints ($R_{\text{penalty}}$).
The reward for a state $s$ is defined as:
\begin{equation}
\label{eq:reward}
R(s) = R_{\text{task}}(s) 
+ \lambda \cdot R_{\text{heuristic}}(s, \Phi)
+ \gamma \cdot R_{\text{penalty}}(s),
\end{equation}
%\yang{$\gamma$ is already used in Section 3. fix notations.}
%\xiangsen{Fixed.}
This composite reward formulation enables MCTS to navigate sparse reward landscapes by exploiting dense heuristic feedback provided by the cold start mechanism. 
$R_{\text{task}}(s)$ captures the primary task objective, such as validation performance (RMSE or AUC) for prediction tasks or molecular property scores (e.g., QED, LogP). 
$R_{\text{heuristic}}(s, \Phi)$ measures the symbolic alignment with the cold-start priors. We formulate this term as the cumulative execution score of the generated symbolic rules $\Phi$ in the cold start phase. 
For a set of symbolic functions $\Phi = \{\phi_1, \dots, \phi_k\}$, the alignment score is defined as the weighted sum of satisfied rules:
\begin{equation}
    \label{eq:alignment}
    R_{\text{heuristic}}(s, \Phi) = \sum_{\phi_i \in \Phi} \mathbb{I}(\phi_i(s) = \text{True}),
\end{equation}
where $\mathbb{I}(\cdot)$ is the indicator function, $\phi_i(s)$ represents the execution of the $i$-th Python rule generated by the LLM (e.g., checking for specific pharmacophores).
% $R_{\text{heuristic}}(s, \Phi)$ measures alignment with the cold start symbolic priors. For optimization, this term is computed by executing the symbolic rules $\phi_i \in \Phi$ on the molecule (e.g., detecting desired pharmacophores).\yang{still unclear, write out the exact equation of "alignment". This looks like an important part to link 4.1 and 4.2 together, so that 4.1 can be considered as an intermediate module in 4.2 as well} For prediction, it reflects the interpretability and relevance of the selected features. 
$R_{\text{penalty}}(s)$ enforces constraints, such as penalizing low structural similarity to the initial molecule (to preserve the scaffold) or chemically invalid structures. $\lambda$ and $\gamma$ are hyperparameters that balance data-driven signals and symbolic priors.

\textbf{Backpropagation.}
The computed reward $R(s)$ is propagated back through the search tree, updating the $Q$-values and visit counts of all ancestor nodes.
This feedback mechanism allows the search to progressively concentrate on the most promising regions of the chemical space, as identified through the synergy between LLM reasoning and rigorous simulation.

\section{Experiment}
\subsection{Experimental Settings}
\textbf{Datasets and Tasks.} We evaluate MolEvolve on two fundamental molecular tasks: (1) \textbf{Property Prediction}, utilizing five benchmarks from MoleculeNet \cite{wu2018moleculenet}, including regression tasks (ESOL, Lipophilicity) and classification tasks (BACE, BBBP, HIV). (2) \textbf{Property Optimization}, using the ChemCoTBench \cite{li2025chemcotbench} benchmark to evaluate the model's ability to improve molecular properties. Results are averaged over 3 independent runs. We follow \cite{guo2023can} and split the dataset into an 80/10/10\% for training, validation and test sets. 

\textbf{Baselines.} We compare our approach against three categories of methods. \textbf{GNN-based Methods:} Specialized graph models including GIN \cite{xu2018gin} and Graphformer \cite{yang2021graphformers}. \textbf{LLM Baselines:} Including RAG-based baselines, like Automolco \cite{zhang2025automated}, MolRAG \cite{xian2025molrag}, and fine-tuned LLMs like Attrilens-mol \cite{lin2025attrilens}. \textbf{Strong Reasoning Models:} Evaluation of the internal reasoning capabilities of proprietary and open-source LLMs.
%, such as o3-mini, DeepSeek-R1, and GPT-4o. 

\textbf{Metrics.}
%For the property prediction task, we adopt standard metrics tailored to the nature of the target variables: for regression tasks, we use the Root Mean Square Error (RMSE) to measure the deviation between predicted and ground-truth values, where a lower RMSE indicates higher precision; for classification tasks, we utilize the ROC-AUC to evaluate the discriminative ability, with higher values representing better performance. 
For property prediction tasks, we adopt Root Mean Square Error (RMSE) (lower value for better performance) for regression and ROC-AUC (higher value for better performance) for classification. 
For the property optimization task, we measure through two metrics: Improvement ($\Delta$) and Success Rate (SR\%). The Improvement metric calculates the average absolute increase in the target property achieved by the model starting from the initial molecule. The SR\% represents the percentage of test cases where the model successfully generates a valid molecule with a property value higher than that of the starting molecule. %This dual-metric approach ensures a balanced evaluation of the model's optimization magnitude and its consistency in navigating structural constraints.

\subsection{Main Results}
\begin{table*}[ht]
\centering
\caption{Comparative results for molecular property prediction on MoleculeNet benchmarks. For regression tasks (ESOL, Lipophilicity), we report RMSE (lower is better). For classification tasks (BACE, BBBP, HIV), we report ROC-AUC (higher is better). }
\label{tab:prediction_results}
\footnotesize
\begin{tabular}{llccccc}
\toprule
\multirow{2}{*}{\textbf{Base Model}} & \multirow{2}{*}{\textbf{Method}} & \multicolumn{2}{c}{\textbf{Regression (RMSE $\downarrow$)}} & \multicolumn{3}{c}{\textbf{Classification (ROC-AUC $\uparrow$)}} \\
\cmidrule(lr){3-4} \cmidrule(lr){5-7}
& & \textbf{ESOL} & \textbf{Lipo.} & \textbf{BACE} & \textbf{BBBP} & \textbf{HIV} \\
\midrule
\multirow{4}{*}{Qwen2.5-7B-Instruct} & Automolco \cite{zhang2025automated} & 1.0566 & 1.0854 & 0.6716 & 0.6102 & 0.6265 \\
& MolRAG \cite{xian2025molrag} & 1.7160 & 2.5412 & 0.6968 & 0.5608 & 0.6476 \\
& ATTRILENS-MOL \cite{lin2025attrilens} & 1.8420 & - & 0.6102 & 0.5810 & - \\
& \textbf{MolEvolve} & \textbf{0.6890} & \textbf{0.8527} & \textbf{0.7705} & \textbf{0.7461} & \textbf{0.7581} \\
\midrule
\multirow{3}{*}{GPT-4o-mini} & Automolco \cite{zhang2025automated} & 1.0062 & 1.0630 & 0.7390 & 0.6435 & 0.6583 \\
& MolRAG \cite{xian2025molrag} & 1.2355 & 1.7800 & 0.7762 & 0.6392 & 0.6823 \\
& \textbf{MolEvolve} & \textbf{0.6914} & \textbf{0.8055} & \textbf{0.7854} & \textbf{0.7268} & \textbf{0.7281} \\
\midrule
\multirow{3}{*}{GPT-4o} & Automolco \cite{zhang2025automated} & 0.8538 & 1.0591 & 0.7310 & 0.6637 & 0.6750 \\
& MolRAG \cite{xian2025molrag} & 1.4237 & 1.2967 & 0.7663 & 0.6271 & 0.6209 \\
& \textbf{MolEvolve} & \textbf{0.6869} & \textbf{0.8188} & \textbf{0.7875} & \textbf{0.7291} & \textbf{0.7682} \\
\midrule
\multirow{2}{*}{GNNs} & GIN \cite{xu2018gin} & 1.2182 & 1.5481 & 0.7326 & 0.6953 & 0.7450 \\
& Graphformer \cite{yang2021graphformers} & 0.9260 & 1.2367 & 0.7695 & 0.7012 & 0.7788 \\
\midrule
\multirow{2}{*}{Tree-based model} & XGBoost & 0.9364 & 0.9833 & 0.6428 & 0.6864 & 0.7020 \\
& Decision Tree & 0.8287 & 1.1135 & 0.6352 & 0.6037 & 0.5566 \\
\midrule
{Linear model} & Logistic Regression (LR) & 0.9970 & 1.0651 & 0.6334 & 0.6727 & 0.6571 \\
\bottomrule
\end{tabular}
\end{table*}

\begin{table}[t] 
    \centering
    \caption{Molecular property optimization results on ChemCoTBench. Improvement measures the average increase in target property, and SR\% (Success Rate) measures the percentage of instances with property improvement. In MolEvolve, the \textbf{Base Model} column specifies the backbone LLM integrated within the architecture.}
    \label{tab:optimization_results}
    \resizebox{\linewidth}{!}
    {
        \setlength{\tabcolsep}{3.5pt} %
        \begin{tabular}{lcccc}
            \toprule
            \multirow{2}{*}{\textbf{Base Model}} & \multicolumn{2}{c}{\textbf{LogP}} & \multicolumn{2}{c}{\textbf{QED}} \\
            \cmidrule(lr){2-3} \cmidrule(lr){4-5}
             & \textbf{$\Delta$} & \textbf{SR\%} & \textbf{$\Delta$} & \textbf{SR\%} \\
            \midrule
            
            % --- MolEvolve Section ---
            \multicolumn{5}{c}{\textit{\textbf{MolEvolve (Ours)}}} \\
            \midrule
            GPT-4o-mini & \textbf{1.925} & \textbf{85\%} & \textbf{0.264} & \textbf{55\%} \\
            Qwen2.5-7B-Instruct & \textbf{1.271} & \textbf{63\%} & \textbf{0.245} & \textbf{58\%} \\
            Qwen2.5-32B-Instruct & \textbf{2.126} & \textbf{82\%} & \textbf{0.258} & \textbf{60\%} \\
            
            % --- Baselines Section ---
            \midrule
            \multicolumn{5}{c}{\textit{\textbf{Baseline Models}}} \\
            \midrule
            %o3-mini & 0.302 & 70\% & 0.287 & 61\% \\
            GPT-5 & 0.257 & 53\% & 0.204 & 66\% \\
            Qwen3-235B-A22B-think & 0.050 & 40\% & 0.020 & 24\% \\
            DeepSeek-R1 & 0.371 & 76\% & 0.074 & 71\% \\
            DeepSeek-V3 & 0.102 & 28\% & 0.090 & 46\% \\
            BioMistral-7B \cite{labrak2024biomistral} & -0.360 & 18\% & -0.295 & 15\% \\
            BioMedGPT-7B \cite{luo2023biomedgpt} & 0.012 & 2\% & 0.000 & 0\% \\
            \bottomrule
        \end{tabular}
    }
\end{table}

\textbf{Property Prediction.}
The results in Table \ref{tab:prediction_results} illustrate the predictive performance of MolEvolve across diverse tasks. MolEvolve consistently achieves a significant reduction in error rates for regression tasks. On the ESOL dataset, MolEvolve maintains an RMSE around 0.69, which is much lower than that of LLM-based baselines like MolRAG \cite{xian2025molrag}. 
Baselines suffer from numerical hallucinations due to the representation-precision paradox.
Their reliance on semantic approximations over-smooths the molecular manifold, which may trigger giant bias in chemical manifolds. 
We resolve this by grounding executable code. By generating deterministic code rather than stochastic values, our framework ensures predictions adhere to deterministic chemical logic.
%These baselines often suffer from numerical hallucinations because LLMs treat continuous physical properties as discrete tokens. Our approach solves this representation-precision paradox. It translates molecular insights into executable code instead of predicting numbers directly. This strategy ensures that the final predictions follow deterministic chemical logic. 
Notably, MolEvolve with Qwen2.5-7B-Instruct outperforms GNNs such as Graphformer. This outcome indicates that the symbolic features discovered through search are more effective than high-dimensional black-box embeddings. In classification tasks, our method also demonstrates strong performance on diverse datasets. %The steady gains across different base models prove the robustness of our evolutionary search mechanism.

\textbf{Property Optimization.}
Table \ref{tab:optimization_results} presents the performance of MolEvolve in molecular optimization tasks. Our method shows a substantial advantage over strong LLMs in the optimization tasks. 
For LogP optimization, MolEvolve achieves an improvement score nearly 4 times higher than that of DeepSeek-R1 with Qwen2.5-7B-Instruct as the base model. This significant gap reveals the limitations of implicit reasoning of LLMs. 
Reasoning LLMs, like DeepSeek-R1, rely on their internal thinking path. They lack the ability to backtrack from failed structural modifications in the vast combinatorial chemical space. In contrast, our MCTS engine performs explicit lookahead planning. It systematically explores multiple search branches and verifies each step through external tools. %Our framework effectively transforms the optimization problem into a manageable evolutionary search process.
A key observation is that our search-centric approach effectively bridges the capability gap between LLMs of different scales. With the full MCTS framework, the Qwen2.5-7B-Instruct ($\Delta=1.271$ and Success Rate 63\% on LogP) is able to outperform strong LLMs like GPT-5 and DeepSeek-V3.
%like o3-mini. 
This result demonstrates that MolEvolve effectively transforms the optimization problem into a manageable evolutionary search process.

\begin{figure}[htbp]
    \centering
    \includegraphics[width=\linewidth]{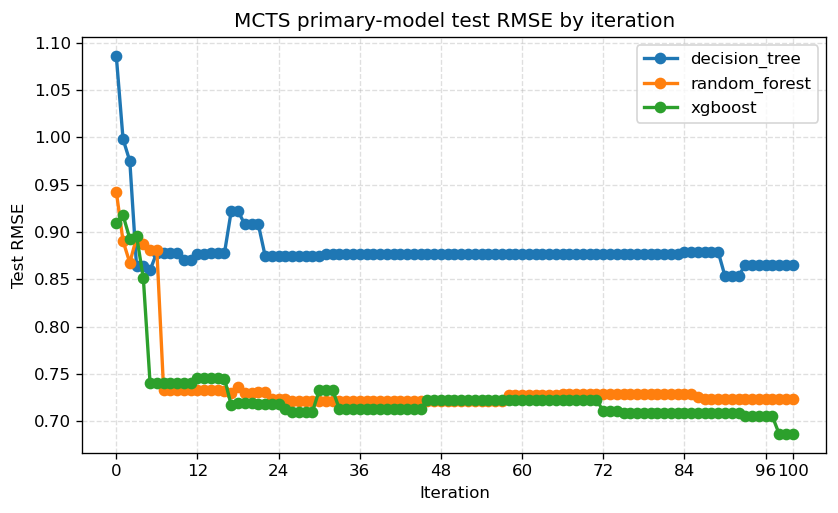}
    \caption{Search Efficiency and Model Adaptability Analysis: Test RMSE trajectories across 100 iterations. To verify that the evolved symbolic features are model-agnostic, we employ three distinct primary models as downstream evaluators. In each iteration, the MCTS engine generates a candidate feature set, which is then frozen to train these evaluators for scoring.}
    \label{fig:mcts_convergence}
\end{figure}

\textbf{Analysis of Search Progress.}
%\yang{this part should go in to 5.2, after  the table.}\xiangsen{Fixed.} 
To further investigate the effectiveness of the searching module of MolEvolve, we visualize the evolutionary trajectories for both property prediction (Figure \ref{fig:mcts_convergence}) and molecular optimization (Figure \ref{fig:logp_convergence}).
The trajectories are across $T=100$ and $T=30$ MCTS iterations, respectively.
% we visualize the test RMSE trajectory across $T=100$ MCTS iterations (see Figure \ref{fig:mcts_convergence}).
%\yang{add one figure for molecular optimization task}). \xiangsen{Added.}
The test set is held out and trajectory analyses report the performance of the optimized model on the held-out test set only after the search concludes. 
%Crucially, we distinguish between two symbolic concepts: the symbolic priors ($\mathcal{K}_{\text{prior}}$), which are the initial heuristic rules distilled from the LLM-driven cold start, and the symbolic feature set ($\mathcal{F}$), which is the dynamic set of executable descriptors accumulated by the MCTS engine.
To demonstrate that the accumulated knowledge is model-agnostic, we employ Random Forest, XGBoost, and Decision Tree as primary evaluators in property prediction. 
%To demonstrate that the accumulated feature sets $\mathcal{F}$ is model-agnostic, we use several machine learning models, including random forest, xgboost and decision tree as primary evaluator models to strictly assess the quality of the discovered symbolic priors generated by the MCTS engine. \yang{BTw, there are many similar terms defined as "symbolic": "symbolic features" "symbolic priors", which are confusing. Define clearly each term using notations such as $\mathcal{F}$, and be consistent! e.g., what are exactly "symbolic priors"?}
The results of both tasks demonstrate a rapid initial convergence in the first steps, which is facilitated by the LLM-driven cold start, and distinct 'step-wise' improvements. It distills essential chemical priors into an initial set of executable rules, providing a high-quality starting manifold.
For \textbf{Property Prediction}, the RMSE trajectory an exploration-breakthrough dynamic over 100 iterations. 
The initial plateau phase (e.g., iterations 20-90) corresponds to the MCTS engine navigating the combinatorial rule space (exploring $\sim 30$ candidates per step) to accumulate necessary feature precursors. 
Unlike greedy methods, MCTS tolerates this latency to identify synergistic combinations. 
The subsequent drop in RMSE signifies a synergistic breakthrough, where the interaction of a newly discovered feature with the accumulated set $\mathcal{F}$ (typically $|\mathcal{F}| \approx 20$) successfully resolves the "activity cliff".
The \textbf{Molecular Optimization} trajectory (Figure \ref{fig:logp_convergence}) displays a more rapid optimization pattern. 
All backbone LLMs exhibit a steep initial ascent, which confirms that distilled rules from Cold Start phase
%($\mathcal{K}_{\text{prior}}$) 
provide immediate, valid structural directives rather than random perturbations. 
Compared to property prediction task, convergence occurs early ($\sim 30$ iterations) in molecular optimization. 
It indicates that constructive structural editing yields more direct feedback than the combinatorial feature discovery required for prediction. 

\begin{figure}[htbp]
    \centering
    \includegraphics[width=\linewidth]{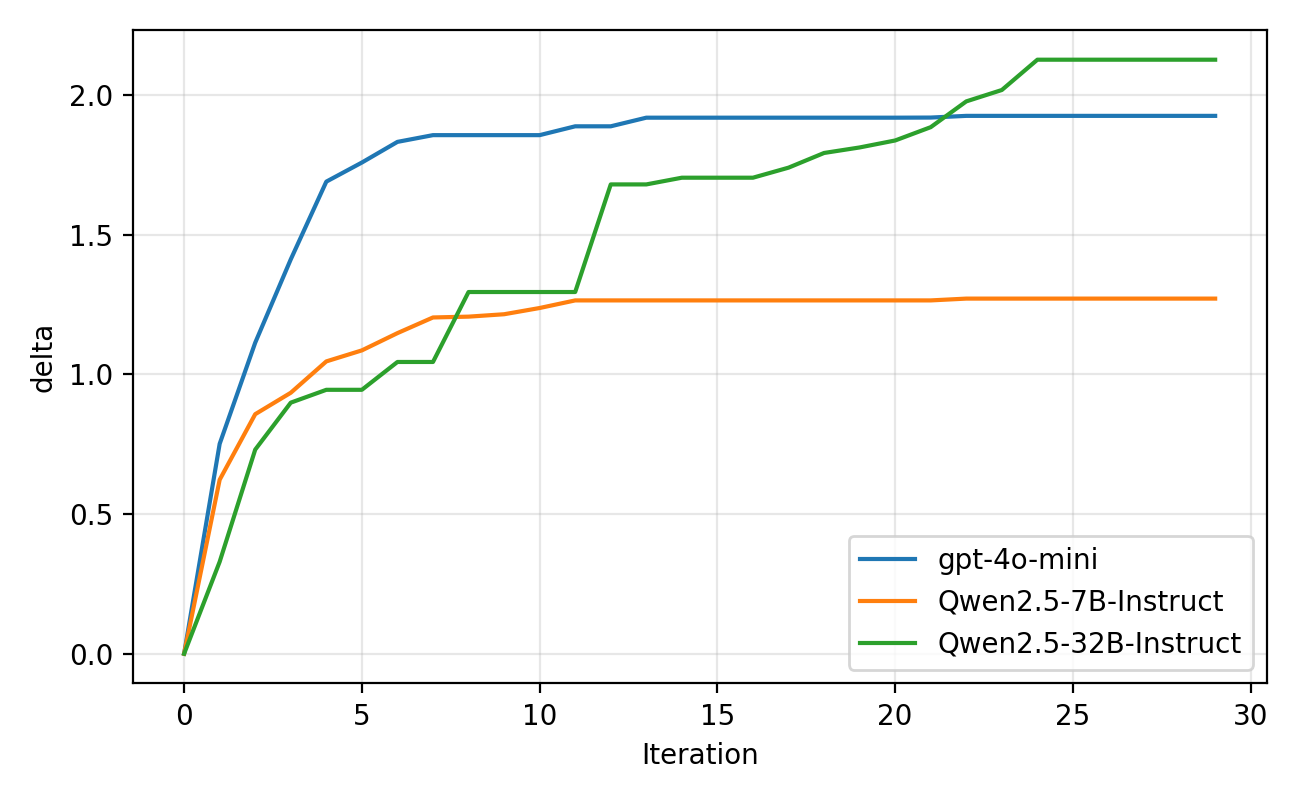}
    \caption{An Example of Evolutionary Trajectory of Molecular Optimization(logP).}
    \label{fig:logp_convergence}
\end{figure}

\textbf{Case Study of Real-Time Explanation.}
%\textcolor{red}{[Todo]}
%\textcolor{blue}{Add a figure illustration here.}
%\yang{we should add an example for "activity cliff" prediction too, otherwise the motivation and theoretical analysis has no direct support. Is the example in Figure 1 on QED a real example that can be demonstrated here? }\xiangsen{Added an example for case study in the Appendix of QED.}
To address the theoretical challenge of "activity cliffs" and verify MolEvolve's ability to navigate complex chemical manifolds, we analyze the optimization of property logP and QED (see Figure~\ref{fig:case_study_trace} and Figure~\ref{fig:qed_full_trace} of Appendix~\ref{appendix:case_study_details} for the detailed trace). 
First, we analyze the optimization of LogP (Figure~\ref{fig:case_study_trace}). Initially, the cold-start phase identifies structural liabilities, and distills symbolic prescription. 
As illustrated in Figure~\ref{fig:case_study_trace}, our MCTS engine systematically explores the structural space. 
%Unlike vanilla LLMs that often 
Unlike proposing chemically invalid molecules or stuck in local optima, MolEvolve employs the Function Executor to prune non-physical paths. The finalized molecule, \texttt{CC=C1NC(C(C)C)C(=O)N(C)C1Cc1ccc(OC)cc1C}, demonstrates a drastic improvement of $\Delta \text{LogP} = 3.60$, nearly 5 times higher than the improvement achieved by deep-reasoning LLMs like DeepSeek-R1. 
This trajectory confirms that MolEvolve follows chemical logic with real-time traceable interpretability.
In addition, we examine the search trajectory for QED (Figure~\ref{fig:qed_full_trace}). 
Our MCTS engine employs explicit look-ahead planning that prevents stagnating in the local optima.
As demonstrated in the trace, the search initially enters a sub-optimal branch. 
Then it executes a backpropogation to the parent node and identifies a synergistic replacement (Nitro $\to$ Isopropyl). 
This symbolic edit successfully resolving the property bottleneck (QED $\to$ 0.831), demonstrating it triggers the "activity cliff". 
This case reinforces our central claim again that, by grounding reasoning in an executable search tree, MolEvolve effectively decouples physical precision from semantic smoothness and offers \textit{real-time interpretability} that every structural edit corresponds to a specific heuristic rule, and creating a transparent trail for chemists.
%Crucially, this process offers \textit{ad-hoc interpretability}: every structural edit corresponds to a specific heuristic rule, and creating a transparent audit trail for chemists. This case reinforces our central claim: by grounding LLM reasoning into an executable and searchable symbolic space, we effectively resolve the representation-precision paradox.

\subsection{Ablation Studies}
%In this section, we conduct comprehensive ablation studies to decouple the performance gains contributed by hyperparameters and different modules.
%across both prediction and optimization tasks.
\begin{table}[h]
    \centering
    \caption{Ablation study for molecular property prediction.}
    \label{tab:prediction_ablation}
    \small
    {
        \setlength{\tabcolsep}{2.5pt} 
        \begin{tabular}{lccccc}
            \toprule
            \multirow{2}{*}{\textbf{Module}} & \multicolumn{2}{c}{\textbf{RMSE $\downarrow$}} & \multicolumn{3}{c}{\textbf{ROC-AUC $\uparrow$}} \\
            \cmidrule(lr){2-3} \cmidrule(lr){4-6}
             & \textbf{ESOL} & \textbf{Lipo} & \textbf{BACE} & \textbf{BBBP} & \textbf{HIV} \\
            \midrule
            
            % --- Qwen2.5-7B ---
            \multicolumn{6}{c}{\textit{\textbf{Base Model: Qwen2.5-7B-Instruct}}} \\
            \midrule
            Vanilla LLM & 1.984 & 2.251 & 0.590 & 0.377 & 0.621 \\
            w/o Cold Start + MCTS & 0.831 & 0.959 & 0.671 & 0.691 & 0.702 \\
            \textbf{w. cold start + MCTS} & \textbf{0.689} & \textbf{0.853} & \textbf{0.771} & \textbf{0.746} & \textbf{0.758} \\
            
            % --- GPT-4o-mini ---
            \midrule
            \multicolumn{6}{c}{\textit{\textbf{Base Model: GPT-4o-mini}}} \\
            \midrule
            Vanilla LLM & 1.873 & 1.763 & 0.702 & 0.443 & 0.677 \\
            w/o Cold Start + MCTS & 0.820 & 0.922 & 0.739 & 0.678 & 0.689 \\
            \textbf{w. cold start + MCTS} & \textbf{0.691} & \textbf{0.806} & \textbf{0.785} & \textbf{0.727} & \textbf{0.728} \\
            
            % --- GPT-4o ---
            \midrule
            \multicolumn{6}{c}{\textit{\textbf{Base Model: GPT-4o}}} \\
            \midrule
            Vanilla LLM & 1.830 & 1.524 & 0.489 & 0.499 & 0.514 \\
            w/o Cold Start + MCTS & 0.737 & 0.882 & 0.750 & 0.681 & 0.714 \\
            \textbf{w. cold start + MCTS} & \textbf{0.687} & \textbf{0.819} & \textbf{0.788} & \textbf{0.729} & \textbf{0.768} \\
            \bottomrule
        \end{tabular}
    }
    \vspace{0.2cm}
    \caption{Ablation study on ChemCoTBench optimization task. $\Delta$ and SR\% represent the average property improvement and success rate, respectively.}
    \label{tab:ablation_study_chemcot}
    {
        \setlength{\tabcolsep}{3.5pt} 
        \small
        \begin{tabular}{lcccc}
            \toprule
            \multirow{2}{*}{\textbf{Modules}} & \multicolumn{2}{c}{\textbf{LogP}} & \multicolumn{2}{c}{\textbf{QED}} \\
            \cmidrule(lr){2-3} \cmidrule(lr){4-5}
             & \textbf{$\Delta$} & \textbf{SR\%} & \textbf{$\Delta$} & \textbf{SR\%} \\
            \midrule
            
            % --- GPT-4o-mini ---
            \multicolumn{5}{c}{\textit{\textbf{Base Model: GPT-4o-mini}}} \\
            \midrule
            Vanilla LLM & 0.148 & 24\% & 0.186 & 31\% \\
            + Cold Start & 0.780 & 42\% & 0.254 & 36\% \\
            \textbf{+ Full MCTS} & \textbf{1.925} & \textbf{85\%} & \textbf{0.264} & \textbf{55\%} \\
            
            % --- Qwen2.5-7B ---
            \midrule
            \multicolumn{5}{c}{\textit{\textbf{Base Model: Qwen2.5-7B-Instruct}}} \\
            \midrule
            Vanilla LLM & -0.662 & 11\% & 0.146 & 28\% \\
            + Cold Start & -0.551 & 13\% & 0.178 & 35\% \\
            \textbf{+ Full MCTS} & \textbf{1.271} & \textbf{63\%} & \textbf{0.245} & \textbf{58\%} \\
            
            % --- Qwen2.5-32B ---
            \midrule
            \multicolumn{5}{c}{\textit{\textbf{Base Model: Qwen2.5-32B-Instruct}}} \\
            \midrule
            Vanilla LLM & -0.032 & 17\% & 0.024 & 15\% \\
            + Cold Start & 0.031 & 37\% & 0.138 & 51\% \\
            \textbf{+ Full MCTS} & \textbf{2.126} & \textbf{82\%} & \textbf{0.258} & \textbf{60\%} \\
            \bottomrule
        \end{tabular}
    }
\end{table}

\textbf{Necessity of Cold Start and Evolutionary Search.}
The comprehensive ablation studies (Table \ref{tab:prediction_ablation} and Table \ref{tab:ablation_study_chemcot}) demonstrate substantial performance gains of different modules across both property prediction and optimization tasks. 
A full analysis of each benchmark is demonstrated in Appendix~\ref{appendix:detail_ablation}.
Vanilla LLMs consistently struggle across both prediction and optimization tasks, which fall short in the \textit{representation-precision paradox}. 
%Their continuous semantic embeddings inherently over-smooth the target molecule manifold, which fails to capture "activity cliffs".
The integration of the Cold Start phase effectively mitigates this by anchoring the model with domain-specific symbolic priors, transforming stochastic hallucinations into deterministic, executable calculations.
However, the most profound performance gains are driven by the Full MCTS framework.
The results indicate that while Cold Start phase provides a high-quality starting manifold, they are insufficient on their own.
Navigating the combinatorial chemical space requires the MCTS engine’s capacity for explicit look-ahead planning and backtracking.
Ultimately, MolEvolve succeeds by decoupling physical precision from semantic generation, which is governed by verifiable evolutionary search rather than random sampling.

\textbf{Analysis of Hyperparameters.}
%Table~\ref{tab:hyperparams}.
To understand the relationship between symbolic reasoning and structural exploration, we conduct a sensitivity analysis on the two key hyperparameters governing the reward function: the \textbf{heuristic guidance weight ($\lambda$)}, which scales the influence of LLM-generated symbolic priors, and the \textbf{similarity penalty weight ($\gamma$)}, which penalizes the structural deviation from the source scaffold. The results for LogP and QED optimization are presented in Figure~\ref{fig:hyperparams_heatmap}.
\begin{figure}[htbp]
    \centering
    \includegraphics[width=\linewidth]{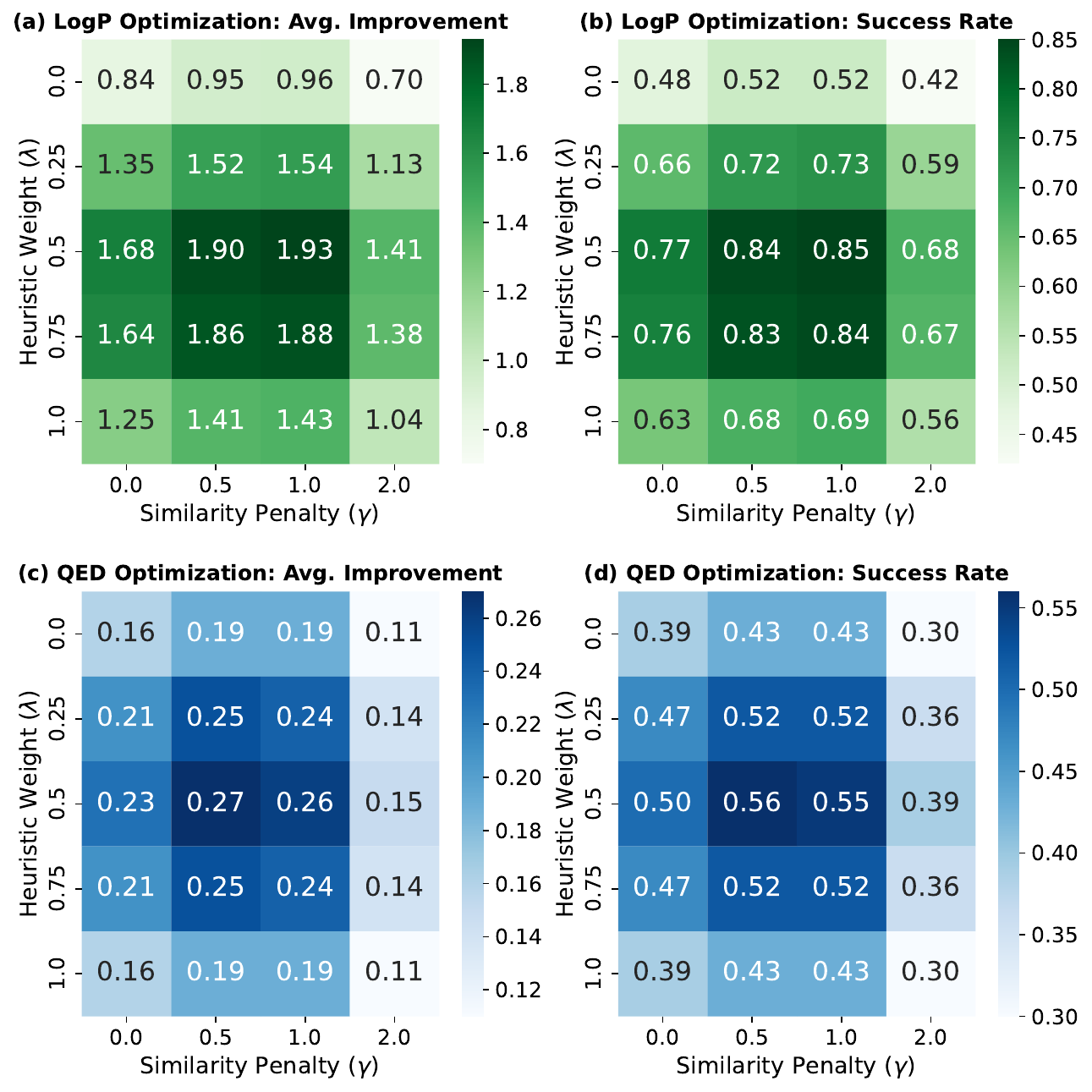} 
    \caption{\textbf{Hyperparameter Sensitivity Analysis Heatmaps.} We visualize the \textit{Average Improvement} (Left) and \textit{Success Rate} (Right) for LogP (Top) and QED (Bottom) tasks. Darker colors indicate superior performance.}
    \label{fig:hyperparams_heatmap}
\end{figure}
We observe a distinct convex trend across both tasks as $\lambda$ increases. 
At low $\lambda$ values, the system degenerates into a quasi-random search, which lacking the high-level planning required to navigate complex chemical spaces. 
Conversely, an excessively high $\lambda$ forces the MCTS to rigidly adhere to LLM-generated heuristics. This effectively suppresses the capability, reducing the system to a greedy follower of the LLM's potentially hallucinated advice. 
Overall, the optimal strategy is a synergy where the LLM provides broad strategic direction, while the MCTS engine retains sufficient autonomy to verify and refine these suggestions against physical reality.
The similarity penalty $\gamma$ acts as a regularizer to balance exploration and scaffold preservation. 
%The similarity penalty $\gamma$ acts as a crucial driver for escaping local optima, which is a necessity for addressing the "activity cliff" challenge. 
Without this penalty, the model tends to make trivial modifications that preserve the scaffold, but fail to trigger "activity cliff". 
However, this is a trade-off; an aggressive penalty ($\gamma=2.0$) imposes an excessive cost on structural modifications, and causes a sharp drop in success rate. 
% However, this is a trade-off; an aggressive penalty ($\gamma=2.0$) prunes too many chemically potential valid candidates \yang{again, without the clear definition of the penalty term, it is hard to understand why this "prunes too many valid candidates".}\xiangsen{Modified.} and causes a sharp drop in success rate. 
Interestingly, QED optimization task favors a slightly less constraint compared to LogP optimization task. This phenomenon can be attributed to reflecting the higher sensitivity of drug-likeness scores to more drastic structural modifications.

\section{Conclusion}
We introduce MolEvolve that reformulates molecular optimization as a dynamic search over an explicit, executable symbolic space. By distilling chemical priors via a cold-start mechanism and rigorously verifying evolved rules through a closed-loop search engine, our approach effectively bridges the gap between the semantic manifold of LLMs and the manifold of molecular properties.
Empirical results across property prediction and optimization tasks demonstrate that MolEvolve not only outperforms baselines but also provides human-readable chemical insights. This shift from implicit neural representations to explicit symbolic evolution offers a new way for scientific discovery. Future work will extend this paradigm to multi-objective AI drug design and explore the integration of laboratory automation.

\section*{Impact Statement}
This paper highlights that our proposed MolEvolve effectively addresses the Representation-Precision Paradox in AI-driven automated molecular discovery. By leveraging the symbolic priors and MCTS look-ahead planning, MolEvolve outperforms existing black-box deep learning paradigms with transparent, human-readable reasoning chains. This shift from implicit continuous fitting to explicit symbolic operations fosters greater trust in AI-assisted decision-making.

\nocite{langley00}

\bibliography{example_paper}
\bibliographystyle{icml2026}

%%%%%%%%%%%%%%%%%%%%%%%%%%%%%%%%%%%%%%%%%%%%%%%%%%%%%%%%%%%%%%%%%%%%%%%%%%%%%%%
%%%%%%%%%%%%%%%%%%%%%%%%%%%%%%%%%%%%%%%%%%%%%%%%%%%%%%%%%%%%%%%%%%%%%%%%%%%%%%%
% APPENDIX
%%%%%%%%%%%%%%%%%%%%%%%%%%%%%%%%%%%%%%%%%%%%%%%%%%%%%%%%%%%%%%%%%%%%%%%%%%%%%%%
%%%%%%%%%%%%%%%%%%%%%%%%%%%%%%%%%%%%%%%%%%%%%%%%%%%%%%%%%%%%%%%%%%%%%%%%%%%%%%%
\newpage
\appendix
\onecolumn

\section{Formal Definitions of MCTS Components}
\label{appendix:mcts_definitions}
%\textcolor{blue}{add description}
In this section, we provide the rigorous mathematical formulation of our LLM-Augmented MCTS framework. In this paper, MolEvolve unifies two distinct tasks: \textit{Property Prediction} and \textit{Property Optimization}. To deploy MCTS for molecular discovery, we define the search environment: the State Space $\mathcal{S}$, Action Space $\mathcal{A}$, and the Reward Mechanism $\mathcal{R}$ that guide the tree search. It constructs a specific MCTS formulation for each, as detailed below. The proposed search process (Algorithm \ref{alg:general_mcts}) follows a four-stage cycle.
In each iteration, the search navigates to a leaf node via UCT (Line 4) and employs the LLM as a policy prior to propose high-potential symbolic actions (Line 6). Crucially, rigorous verification (Line 8) serves as a filter: only chemically valid states trigger reward computation and backpropagation, while invalid hallucinations are immediately pruned via heavy penalties.

\begin{table*}[htbp]
\centering
\caption{Formal definitions of MCTS components across molecular tasks in our experiments.}
\label{tab:mcts_definition_appendix}
\footnotesize
\begin{tabular}{p{2.2cm}p{6.3cm}p{6.3cm}}
\toprule
\textbf{Component} & \textbf{Property Prediction (Feature Evolution)} & \textbf{Property Optimization (Structural Evolution)} \\
\midrule
\textbf{State ($s$)} & A tuple of symbolic feature names $\mathcal{F} = (f_1, f_2, \dots, f_k)$ representing a selected descriptor set. & A validated molecular SMILES string $S$ representing a specific chemical structure. \\
\textbf{Action ($a$)} & Selection of a novel RDKit-computable descriptor $f'$ from the evolved pool to augment the feature set $\mathcal{F} \rightarrow \mathcal{F}'= (f_1, f_2, \dots, f_k, f')$. & A multi-stage structural transformation $S \rightarrow S'$ involving LLM-driven critique and rewriting. \\
\textbf{Expansion} & LLM identifies high-potential descriptors by analyzing physicochemical relevance and existing feature importance. & LLM acts as molecule operator, performing structural diagnosis and modification proposals. \\
\textbf{Reward ($R$)} & Validation metric (e.g., negative RMSE or AUC) augmented by a normalized test signal bonus for exploration stability. & Weighted sum of task score (e.g., LogP/QED), cold-start heuristic scores, and similarity-based penalties. \\
\textbf{Verification} & Deterministic RDKit-based feature computation; ensures all descriptors in $\mathcal{F}$ are executable. & Enforced by RDKit sanitization and valence checks; invalid SMILES are assigned a heavy penalty ($R=-100$). \\
\bottomrule
\end{tabular}
\end{table*}

\begin{algorithm}[H]
\caption{LLM-guided Evolutionary Search}
\label{alg:general_mcts}
\begin{algorithmic}[1]
\STATE \textbf{Input:} Initial state $s_0$, Textual Rules $\mathcal{K}_{text}$, Symbolic Functions $\Phi$, LLM $\mathcal{M}$, Search Iterations $N_{iter}$
\STATE \textbf{Output:} Optimal state $s^*$
\STATE $\mathbb{T} \leftarrow \{s_0\}$ \COMMENT{Initialize search tree with cold-start state}
\WHILE{Iterations $< N_{iter}$}
    \STATE $s_{leaf} \leftarrow \text{Selection}(\mathbb{T})$ \COMMENT{Maximize UCT value, see Eq.~\ref{eq:uct}} 
    
    \STATE \COMMENT{// Expansion conditioned on textual priors $\mathcal{K}_{text}$}
    \STATE $\mathcal{A}_{\text{LLM}} \leftarrow \mathcal{M}(s_{leaf}, \mathcal{K}_{text})$ \COMMENT{Propose actions matching chemical logic}
    
    \FOR{each proposed action $a \in \mathcal{A}_{\text{LLM}}$}
        \STATE $s_{new} \leftarrow \text{Apply\_Action}(s_{leaf}, a)$
        \STATE \textbf{if} $\text{IsValid}(s_{new})$ is \textbf{True} \textbf{then}
            \STATE \COMMENT{// Simulation with symbolic alignment reward}
            \STATE $R \leftarrow \text{Compute\_Reward}(s_{new}, \Phi)$ \COMMENT{Calculate $R(s)$ using Eq.~\ref{eq:reward}}
            \STATE $\text{Insert } s_{new} \text{ into } \mathbb{T}$
            \STATE $\text{Backpropagation}(s_{new}, R)$ \COMMENT{Update $Q$ and $N$ values}
        \STATE \textbf{else}
            \STATE Assign penalty $R_{min}$ to prune invalid path
        \STATE \textbf{end if}
    \ENDFOR
\ENDWHILE
\STATE \textbf{return} $s^* = \arg\max_{s \in \mathbb{T}} R(s)$
\end{algorithmic}
\end{algorithm}

\section{Theoretical Analysis} 
%\yang{mention this in main text somewhere}
\label{appendix:theoretical_analysis}
In this section, we provide a theoretical foundation for the observed limitations of LLMs in molecular quantitative tasks that proposed in Section~\ref{sec:intro}. 
We argue that the failure stems from a fundamental topological mismatch between the semantic manifold learned by LLMs and the physical manifold required for quantitative prediction.

\subsection{Topological Mismatch: Semantic vs. Euclidean Manifolds}
The Manifold Hypothesis posits that high-dimensional data reside on a low-dimensional manifold $\mathcal{M}$ \cite{modell2025origins}. LLMs, trained with a language modeling objective, map molecular strings (e.g., SMILES) onto a semantic manifold where the metric is defined by contextual similarity (e.g., cosine similarity). 

However, quantitative tasks (e.g., LogP prediction) imply a mapping to a \textbf{Euclidean Manifold} ($\mathbb{R}$), where distance corresponds to absolute physical difference. A critical conflict arises due to the \textit{Smoothness Inductive Bias} of neural networks: LLMs tend to map token-wise similar inputs to proximal embeddings. In chemistry, however, a slight character change could trigger an "activity cliff", a drastic shift in physical properties. This discrepancy creates a smoothing hallucination, where the model interpolates over sharp physical gradients.

\subsection{Formal Analysis of Approximation Error}
To formalize this intuition, we analyze the approximation error using discrete Lipschitz continuity, avoiding the ill-defined use of continuous limits on the discrete molecular space.

\begin{definition}[Discrete Local Lipschitz Constant]
Let $(\mathcal{X}, d_{\mathcal{X}})$ be the discrete metric space of molecules (e.g., equipped with Levenshtein distance). The local Lipschitz constant of a function $f: \mathcal{X} \to \mathbb{R}$ at $x$ over a 1-step neighborhood $\mathcal{N}_1(x) = \{x' \mid d_{\mathcal{X}}(x, x')=1\}$ is defined as:
\begin{equation}
    L_f(x) = \max_{x' \in \mathcal{N}_1(x)} |f(x) - f(x')|.
\end{equation}
\end{definition}

We introduce two key assumptions characterizing the conflict between the ground truth nature and the model bias:

\begin{assumption}[Activity Cliffs vs. Model Smoothness]
\label{ass:cliff_vs_smooth}

\textbf{Chemical Reality:} There exists a subset of molecules $\mathcal{X}_{\text{cliff}} \subset \mathcal{X}$ where the ground truth function $f^*$ exhibits activity cliffs, characterized by a large local change $K$:
    \begin{equation}
        L_{f^*}(x) \ge K, \quad \forall x \in \mathcal{X}_{\text{cliff}}.
    \end{equation}
    
\textbf{Model Bias:} The LLM-based predictor $f_{\theta}$ is constrained by regularization (e.g., LayerNorm, bounded attention) to be smooth, satisfying a much smaller Lipschitz bound $\kappa$:
\begin{equation}
    L_{f_{\theta}}(x) \le \kappa, \quad \text{where } \kappa \ll K.
\end{equation}

\end{assumption}

\begin{theorem}[Lower Bound of Worst-Case Error]
\label{thm:error_bound}
Let $x \in \mathcal{X}_{\text{cliff}}$ be a molecule at an activity cliff. Assume the model perfectly fits the training sample $x$, i.e., $f_{\theta}(x) = f^*(x)$. Then, there exists a neighbor $x' \in \mathcal{N}_1(x)$ such that the prediction error is strictly lower-bounded by the difference in Lipschitz constants:
\begin{equation}
    |f_{\theta}(x') - f^*(x')| \ge K - \kappa.
\end{equation}
\end{theorem}

\begin{proof}
By the definition of $L_{f^*}(x)$, there exists a specific neighbor $x^* \in \mathcal{N}_1(x)$ that realizes the maximum change, such that $|f^*(x) - f^*(x^*)| \ge K$.
For this same neighbor $x^*$, the model's prediction change is upper-bounded by its smoothness constraint: $|f_{\theta}(x) - f_{\theta}(x^*)| \le \kappa$.

We analyze the error at $x^*$, denoted as $E(x^*) = |f_{\theta}(x^*) - f^*(x^*)|$. Using the reverse triangle inequality $|A - B| \ge |A| - |B|$:
\begin{align}
    E(x^*) &= |f_{\theta}(x^*) - f^*(x^*) - \underbrace{(f_{\theta}(x) - f^*(x))}_{0}| \nonumber \\
           &= |(f^*(x^*) - f^*(x)) - (f_{\theta}(x^*) - f_{\theta}(x))| \nonumber \\
           &\ge \underbrace{|f^*(x^*) - f^*(x)|}_{\text{Physical Shift}} - \underbrace{|f_{\theta}(x^*) - f_{\theta}(x)|}_{\text{Semantic Smoothing}} \nonumber \\
           &\ge K - \kappa.
\end{align}
Since $K \gg \kappa$, the error is strictly positive and irreducible under the smoothness constraint.
\end{proof}

\subsection{Resolution: Breaking the Smoothness Barrier via Discrete Search}
\label{sec:theory_bridge}
The inequality derived in Theorem~\ref{thm:error_bound} implies that any purely continuous approximator $f_{\theta}$ is fundamentally limited by its Lipschitz constant $\kappa$. When facing activity cliffs where the physical gradient $K \gg \kappa$, the model inevitably hallucinates a smooth transition. To resolve this, we introduce MCTS to inject runtime non-smoothness into the inference process without altering the model parameters.

\textbf{Mechanism of Decoupling.} 
Unlike direct inference, which is constrained by the fixed geometry of the embedding space, MCTS performs a discrete search over a symbolic action space $\mathcal{A}$. Let the search process be formalized as an operator $\mathcal{T}_N(\cdot)$ with a computational budget of $N$ simulations. By utilizing a deterministic verifier $\mathcal{V}$ (e.g., a computational oracle or rule-based checker) during simulation, we decouple prediction accuracy from the weights $\theta$.

The LLM acts as a policy prior $\pi_{\theta}(a|s)$ to guide the search. As the search budget $N$ increases, the estimator $\hat{y}_{\text{MCTS}}$ converges to the ground truth value reachable within the search depth, effectively bypassing the smoothness bound $\kappa$:
\begin{equation}
    \label{eq:mcts_limit}
    \lim_{N \to \infty} | \hat{y}_{\text{MCTS}}(x) - f^*(x) | = 0.
\end{equation}
This formulation proves that by shifting the burden of precision from the parameter space ($\theta$, constrained by $\kappa$) to the discrete search space ($\mathcal{T}_N$, capable of modeling $K$), we can eliminate the irreducible error derived in Theorem~\ref{thm:error_bound}.

\section{Detailed Ablation Analysis}
\label{appendix:detail_ablation}
\textbf{Mechanism of Symbolic Priors in Prediction.}
As summarized in Table \ref{tab:prediction_ablation}, the transition from vanilla LLMs to our full-pipeline MolEvolve yields a drastic reduction in prediction error. 
For instance, on the regression task like ESOL, all base models initially suffer from high RMSE ($\approx 1.8$), which we attribute to the representation-precision paradox mentioned before. 
LLMs tend to map structurally similar molecules to similar semantic embeddings and fail to capture the "activity cliffs" where minor structural changes trigger non-linear property leaps. 
The integration of the cold-start mechanism further refines the search space; for Qwen2.5-7B, adding symbolic priors improves the RMSE from 0.831 to 0.689.
This improvement suggests that explicit symbolic grounding (e.g., grounding heuristics into executable RDKit code) effectively calibrates the LLM's stochastic predictions into deterministic physical calculations.

\textbf{Overcoming Optimization Challenge via Extensive Search.}
The molecular optimization task on ChemCoTBench exhibits a more pronounced dependency on the MCTS framework. As shown in Table \ref{tab:ablation_study_chemcot}, vanilla LLMs often yield negligible or even negative improvements ($\Delta \approx -0.032$). 
This phenomenon highlights that auto-regressive generation in combinatorial spaces lacks a mechanism to anticipate the property impact of structural modifications. 
Cold Start provides a critical warm-up by injecting domain-specific symbolic knowledge. When provided with cold-start phase, it significantly boost the success rate (SR\%). 
However, the most substantial performance surge is driven by the full-stage modules. For Qwen2.5-32B, the improvement $\Delta$ on LogP leaps from 0.031 to 2.126 upon enabling the full-stage framework. This signifies that while symbolic priors provide a better starting manifold, the MCTS's ability to perform explicit backtracking and value-guided exploration is the primary engine for navigating the high-dimensional chemical landscape.

\section{Extended Case Studies}
\label{appendix:case_study_details}
In this section, we provide a granular analysis of an optimization trajectory illustrated in Figure \ref{fig:case_study_trace}. The trajectory culminates in a finalized structure with a \text{LogP} of 2.51, with an improvement of $\Delta \text{LogP} = 3.60$. This step-by-step trace constitutes an real-time and traceable explanation: the reasoning (the rule) precedes the action (the modification). 
%The explicit explanation ensures that the optimization process is transparent, logical, and chemically grounded.
Figure \ref{fig:case_study_trace} illustrates the optimization of logP of a molecule. The cold start phase (Phase 1) correctly diagnoses chemical rules, like "Polarity Overload", due to phenol and amide groups. 
The MCTS engine then systematically executes the prescribed rules. First, it masks the polar phenol group via O-methylation (Iter 1), then increases hydrophobic bulk via side-chain extension (Iter 2). 
This trace confirms that the system strictly precedes and dictates the action (structural modification), ensuring the process is transparent and logically consistent. 

Figure \ref{fig:qed_full_trace} visualizes how MCTS handles QED task. \textbf{Exploration vs. Exploitation:} The search initially explores a branch (Iter 1) and an alternative branch (Iter 2, changing Nitro to Carboxyl).
\textbf{Escaping Local Optima:} Crucially, Iter 2 results in a sub-optimal state (QED 0.676) where the polarity issue is only partially addressed. Instead of getting stuck in this local optimum (as a greedy approach might), the MCTS engine backtracks to the parent state.
\textbf{Synergistic Breakthrough:} In Iter 3, the engine combines the simplified scaffold from Iter 1 with a high-lipophilicity pharmacophore replacement (Nitro $\to$ Isopropyl), achieving a global optimum (QED 0.831).
This dynamic proves that MolEvolve transcends simple heuristic following; it possesses the planning capability to prune inferior branches and navigate complex chemical manifolds.

\begin{figure*}[htbp]
    \centering
    \tcbset{colback=white, colframe=black!75!white, fonttitle=\bfseries}
    % Cold Start 
    \begin{tcolorbox}[title=\textsc{Phase 1: Cold-Start Symbolic Diagnosis \& Prescription}, colframe=blue!60!black, colback=blue!5!white]
        \small
        \textbf{Input Molecule:} \texttt{O=C1NC(Cc2ccc(O)cc2)C(=O)NC1CO} (Task Score: -1.09)
        \par\noindent\hrulefill\par
        \begin{minipage}[t]{0.48\textwidth}
            \textbf{\textcolor{red}{Diagnosis (Liabilities):}}
            \begin{itemize} \setlength\itemsep{0em}
                \item \textbf{Polarity Overload:} Multiple polar groups ($-\text{OH}$, $-\text{NH}$, $\text{C=O}$) significantly reduce LogP.
                \item \textbf{Aromatic Hydroxyl:} The phenol group on the ring is a major H-bond donor.
                \item \textbf{Rigid Scaffold:} Cyclic amide structure limits lipophilic interactions.
            \end{itemize}
        \end{minipage}
        \hfill
        \begin{minipage}[t]{0.48\textwidth}
            \textbf{\textcolor{green!60!black}{Prescription (Heuristic Rules):}}
            \begin{itemize} \setlength\itemsep{0em}
                \item \textbf{R1 (De-polarization):} Replace aromatic $-\text{OH}$ with $-\text{OCH}_3$ or halogens.
                \item \textbf{R2 (Side-chain):} Extend alkyl chains on amide nitrogens (e.g., $+\text{C}_4\text{H}_9$).
                \item \textbf{R3 (Scaffold Mod):} Reduce carbonyl count or open the ring if feasible.
                \item \textbf{R4 (Substituent):} Add lipophilic groups (alkyl/aryl) to the ring.
            \end{itemize}
        \end{minipage}
    \end{tcolorbox}
    
    \vspace{0.1cm}
    \begin{tcolorbox}[title=\textsc{Phase 2: MCTS Optimal Evolutionary Trajectory (Real-Time Traceable Execution)}, colframe=green!40!black, colback=green!5!white]
        \footnotesize
        \centering
        \begin{tabular}{c p{10cm} c}
            \toprule
            \textbf{Iter} & \textbf{Structural Modification \& Rationale} & \textbf{LogP Score} \\
            \midrule
            \textbf{Start} & \texttt{O=C1NC(Cc2ccc(O)cc2)C(=O)NC1CO} & -1.09 \\
            $\downarrow$ & \textit{Diagnosis: High polarity due to -OH and amides.} & \\
            
            \textbf{Iter 1} & \textbf{Action:} Phenol O-methylation ($-\text{OH} \to -\text{OCH}_3$). \newline
            \textbf{Logic:} Executes \textbf{Rule R1}. Masks the polar hydroxyl group. \newline
            \texttt{C=C1NC(CO)C(=O)NC1Cc1ccc(OC)cc1} & 0.20 \\
            \midrule
            
            \textbf{Iter 2} & \textbf{Action:} Side chain extension (Add Isopropyl group). \newline
            \textbf{Logic:} Executes \textbf{Rule R2}. Increases hydrophobic bulk. \newline
            \texttt{C=C1NC(C(C)C)C(=O)NC1Cc1ccc(OC)cc1} & 1.86 \\
            \midrule
            
            \textbf{Iter 3} & \textbf{Action:} N-methylation on amide. \newline
            \textbf{Logic:} Executes \textbf{Rule R2/R3}. Reduces H-bond donor count. \newline
            \texttt{C=C1NC(C(C)C)C(=O)N(C)C1Cc1ccc(OC)cc1} & 2.20 \\
            \midrule
            
            \textbf{Iter 4} & \textbf{Action:} Aromatic ring methylation. \newline
            \textbf{Logic:} Executes \textbf{Rule R4}. Adds lipophilic substituent. \newline
            \texttt{C=C1NC(C(C)C)C(=O)N(C)C1Cc1ccc(OC)cc1C} & 2.51 \\
            \midrule
            
            \textbf{Final} & \textbf{Best Validated Structure:} \newline
            \texttt{CC=C1NC(C(C)C)C(=O)N(C)C1Cc1ccc(OC)cc1C} & \textbf{2.51} \\
            \bottomrule
        \end{tabular}
        \par\vspace{0.1cm}
        \textbf{Outcome:} Total Improvement $\Delta \text{LogP} = 3.60$. The MCTS explicitly followed the heuristic prescription, validating the real-time and traceable explanation of the optimization process.
    \end{tcolorbox}
    
    \caption{\textbf{Real-time Traceable Optimization Case Study.} The framework first generates a symbolic diagnosis and prescription (Top). The MCTS engine then executes this plan step-by-step (Bottom). Unlike vanilla LLMs which might hallucinate a jump to the final result, our MolEvolve produces a verifiable edit history where each step corresponds to a specific heuristic rule, providing intrinsic traceable explainability.}
    \label{fig:case_study_trace}
\end{figure*}

\definecolor{explore_bg}{RGB}{250, 250, 240}
\definecolor{subopt_bg}{RGB}{255, 245, 235}
\definecolor{best_bg}{RGB}{235, 255, 235} 
\definecolor{prune_bg}{RGB}{245, 245, 245}

\begin{figure*}
    \centering
    \tcbset{colback=white, colframe=black!75!white, fonttitle=\bfseries}
    
    % --- Phase 1: Cold Start ---
    \begin{tcolorbox}[title=\textsc{Phase 1: Cold-Start Symbolic Diagnosis}, colframe=blue!60!black, colback=blue!5!white]
        \small
        \textbf{Input Molecule:} \texttt{O=C(O)c1cn(COCCO)c2ccc([N+](=O)[O-])cc2c1=O} (QED: 0.453)
        \par\noindent\hrulefill\par
        \begin{minipage}[t]{0.48\textwidth}
            \textbf{\textcolor{red}{Liabilities:}}
            \begin{itemize} \setlength\itemsep{0em}
                \item \textbf{Nitro Group:} High polarity (\texttt{[N+](=O)[O-]}).
                \item \textbf{Side Chain:} Hydrophilic ether (\texttt{COCCO}).
            \end{itemize}
        \end{minipage}
        \hfill
        \begin{minipage}[t]{0.48\textwidth}
            \textbf{\textcolor{green!60!black}{Prescription:}}
            \begin{itemize} \setlength\itemsep{0em}
                \item \textbf{R1:} Replace Nitro with lipophilic alkyls (e.g., Isopropyl).
                \item \textbf{R2:} Simplify side chain.
            \end{itemize}
        \end{minipage}
    \end{tcolorbox}
    
    \vspace{0.1cm}
    
    % --- Phase 2: Full MCTS History ---
    \begin{tcolorbox}[title=\textsc{Phase 2: MCTS Optimal Evolutionary Trajectory (Real-time Traceable Execution)}, colframe=green!40!black, colback=white]
        \footnotesize
        \centering
        \begin{tabular}{c p{8.5cm} c l}
            \toprule
            \textbf{Iter} & \textbf{Structural Modification \& Search Logic} & \textbf{QED} & \textbf{Status} \\
            \midrule
            
            % --- Iter 1 ---
            \rowcolor{explore_bg}
            \textbf{1} & 
            \textbf{Branch A (Simplification):} Truncate side chain \texttt{COCCO} $\to$ \texttt{CC(O)}. \newline
            \textit{Rationale:} Testing Rule R2 (reduce MW). \newline
            \texttt{CC(\textcolor{blue}{O})C1=C(C(=O)O)C(=O)c2cc([N+](=O)[O-])ccc21} & 
            0.477 & 
            \textcolor{orange}{Exploring} \\
            \midrule
            
            % --- Iter 2 ---
            \rowcolor{subopt_bg}
            \textbf{2} & 
            \textbf{Branch B (Alternative):} Replace Nitro with Carboxyl \texttt{C(=O)O}. \newline
            \textit{Rationale:} Attempting to fix polarity, but Carboxyl is still polar. \newline
            \texttt{O=C(O)C1=C(C(=O)O)c2ccc(\textcolor{orange}{C(=O)O})cc2C1=O} & 
            0.676 & 
            \textcolor{brown}{Sub-optimal} \\
            \midrule
            
            % --- Iter 3 ---
            \rowcolor{best_bg}
            \textbf{3} & 
            \textbf{Back to Branch A (Breakthrough):} From Iter 1, replace Nitro $\to$ Isopropyl. \newline
            \textit{Rationale:} \textbf{Synergy!} Simplified chain (Iter 1) + Lipophilic Group (Rule R1). \newline
            \texttt{CC(C(=O)O)C1=C(C(=O)O)C(=O)c2cc(\textcolor{green!60!black}{C(C)C})ccc21} & 
            \textbf{0.831} & 
            \textcolor{green!60!black}{\textbf{Optimal}} \\
            \midrule
            
            % --- Iter 4 ---
            \rowcolor{prune_bg}
            \textbf{4} & 
            \textbf{Continue Branch B (Pruning):} Further edit Iter 2. \newline
            \textit{Rationale:} Trying to save the sub-optimal branch, but fails. \newline
            \texttt{Cc1ccc2c(c1)C(=O)C(C(=O)O)=C2C(=O)O} \newline
            Tested QED: 0.745 $<$ 0.831 \newline
            Fallback to \texttt{CC(C(=O)O)C1=C(C(=O)O)C(=O)c2cc(\textcolor{green!60!black}{C(C)C})ccc21}
            & 
            \textbf{0.831} & 
            \textcolor{gray}{Pruned} \\
            \bottomrule
        \end{tabular}
        \par\vspace{0.1cm}
        \textbf{Search Analysis:} The MCTS did not follow a straight line. It first explored side-chain simplification (Iter 1), then tried a sub-optimal polarity fix (Iter 2). Crucially, it \textbf{backtracked} to Iter 1 and applied the correct pharmacophore replacement (Iter 3), achieving the global optimum.
    \end{tcolorbox}
    
    \caption{\textbf{Step-by-Step MCTS Search History.} This detailed view shows how the algorithm explores different branches. Note how \textbf{Iter 3} (the breakthrough) was achieved by combining the structural simplification from \textbf{Iter 1} with the high-level symbolic prescription (Nitro replacement), effectively pruning the inferior path taken in \textbf{Iter 2}.}
    \label{fig:qed_full_trace}
\end{figure*}

\section{Details and Example of Cold Start}
\label{appendix:cold_start_example}
We provide a concrete implementation of the Cold Start phase using the ESOL dataset as an example. Figure~\ref{fig:prompt_p0} outlines the prompt structure $\mathcal{P}_{0}$ used to guide the LLM in synthesizing computable chemical priors. The resulting set of heuristic rules ($\mathcal{K}_{text}$), which covers diverse properties such as polarity and topology to serve as the basis for symbolic grounding, is presented in Figure~\ref{fig:appendix_rules}.

\begin{figure}[h]
    \centering
    \begin{tcolorbox}[
        colback=gray!10, 
        colframe=gray!50, 
        title=\textbf{Prompt $\mathcal{P}_{0}$: Knowledge Synthesis for ESOL}
    ]
    \small
    \textbf{Role:} You are a computational chemistry expert specializing in molecule understadning and feature engineering.\\
    \textbf{Task:} Synthesize prior domain knowledge to predict water solubility (logS) by proposing a diverse set of heuristic rules. \\
    \textbf{Instructions:}
    \begin{enumerate}
        \item Propose 20-30 concise heuristic rules that capture complementary structural or physico-chemical patterns.
        \item \textbf{Constraint 1:} Each rule MUST start with "Calculate..." and focus on properties computable via RDKit.
        \item \textbf{Constraint 2:} Avoid duplicates or mere variations of the same property (e.g., do not list both "Calculate LogP" and "Calculate partition coefficient").
        \item \textbf{Constraint 3:} Keep each rule description short (within 20 words).
        \textbf{Output:} Return a strictly structured JSON object containing the rule descriptions.
    \end{enumerate}
    \end{tcolorbox}
\caption{An example of the prompt $\mathcal{P}_{0}$ used in the Knowledge Synthesis phase for ESOL dataset. It serves as the initial instruction to distill chemical intuition into structured textual rules.}
    \label{fig:prompt_p0}
\end{figure}

\begin{figure}[h]
    \centering
    \begin{tcolorbox}[
        colback=blue!3,        
        colframe=blue!40!black,
        title=\textbf{Heuristic Rules for ESOL},
        fonttitle=\bfseries\small,
        boxrule=0.8pt,
        arc=2pt
    ]
    \small
    \begin{minipage}[t]{0.48\textwidth}
    1. Calculate the number of polar functional groups present in the molecule.\\
    2. Calculate the molecular weight.\\
    3. Calculate the number of hydrogen bond donors in the structure.\\
    4. Calculate the number of hydrogen bond acceptors in the structure.\\
    5. Calculate the presence of ionic groups.\\
    6. Calculate the degree of branching in aliphatic chains.\\
    7. Calculate the presence of aromatic rings.\\
    8. Calculate the overall charge of the molecule.\\
    9. Calculate the logP value.\\
    10. Calculate the presence of hydrophilic substituents.
    \end{minipage}%
    \hfill
    \begin{minipage}[t]{0.48\textwidth}
    11. Calculate the size of the hydrophobic regions.\\
    12. Calculate the presence of functional groups that can ionize.\\
    13. Calculate the steric hindrance around polar groups.\\
    14. Calculate the number of carbon atoms.\\
    15. Calculate the presence of cyclic structures.\\
    16. Calculate the presence of fluorine atoms.\\
    17. Calculate the presence of sulfur or phosphorus.\\
    18. Calculate the molecular symmetry.\\
    19. Calculate the presence of multiple functional groups.\\
    20. Calculate the pKa of acidic or basic groups. \\
    ...
    \end{minipage}
    \end{tcolorbox}
    \caption{An example of textual heuristic rules ($\mathcal{K}_{text}$) generated by the LLM for the ESOL prediction task. These rules serve as the semantic basis for the subsequent symbolic code generation.}
    \label{fig:appendix_rules}
\end{figure}

%%%%%%%%%%%%%%%%%%%%%%%%%%%%%%%%%%%%%%%%%%%%%%%%%%%%%%%%%%%%%%%%%%%%%%%%%%%%%%%
%%%%%%%%%%%%%%%%%%%%%%%%%%%%%%%%%%%%%%%%%%%%%%%%%%%%%%%%%%%%%%%%%%%%%%%%%%%%%%%

\end{document}